\documentclass[twoside,11pt]{article}

\usepackage{journal,times}
\usepackage{float}

\usepackage{latexsym,wrapfig}
\usepackage{graphicx}

\usepackage{float}
\usepackage{latexsym}
\usepackage{amsmath} 
\usepackage{amstext}
\usepackage{amsfonts}
\usepackage{amsopn}
\usepackage{ifthen}
\usepackage[mathscr]{eucal}

\usepackage[dvips,backref,pageanchor=true,plainpages=false,
pdfpagelabels,bookmarks,bookmarksnumbered,breaklinks,
pdfborder={0 0 0},  
]{hyperref}

\newcommand{\ignore}[1]{}

\floatstyle{ruled}
\newfloat{algorithm}{h}{loa}
\floatname{algorithm}{Algorithm}

\newfloat{phase}{h}{loa}
\floatname{phase}{Phase}

\newfloat{subroutine}{h}{loa}
\floatname{subroutine}{Subroutine}

\newcommand{\nt}{\newtheorem}

\nt{note} {Note} \nt{fact} {Fact}

\newcommand{\err}{{\rm err}}

\newcommand{\plur}{\mathrm{plur}}

\newcommand{\true}{\mathrm{true}}

\newcommand{\ps}{\mathrm{ps}}
\newcommand{\s}{\mathrm{s}}

\newcommand{\D}{\mathcal D} 
\newcommand{\Dset}{\mathbb{D}} 
\newcommand{\X}{\mathcal X} 
\newcommand{\DIS}{{\rm DIS}}
\newcommand{\Y}{\mathcal Y} 
\newcommand{\alg}{\mathcal A} 
\renewcommand{\H}{\mathcal H} 
\newcommand{\U}{\mathcal U} 
\newcommand{\C}{\mathbb C} 
\renewcommand{\P}{\mathbb P} 
\newcommand{\nats}{\mathbb{N}} 
\newcommand{\Epsilon}{\mathscr{E}}
\newcommand{\E}{\mathbb E}

\newcommand{\tx}{\tilde{x}}
\newcommand{\ty}{\tilde{y}}
\newcommand{\A}{\cal A}
\newcommand{\state}{{\cal S}}
\newcommand{\RCN}{{\rm RCN}}
\newcommand{\Realizable}{{\mathscr R}{\rm ealizable}}
\newcommand{\BoundedNoise}{{\rm BN}}
\newcommand{\Agnostic}{{\mathscr A}{\rm gnostic}}

\newcommand{\AL}{\mathrm{AL}}
\newcommand{\CCQM}{\mathrm{CCQ}}
\newcommand{\SC}{\mathrm{QC}}

\newcommand{\argmax}{\mathop{\rm argmax}}

\jmlrheading{}{}{}{}{}{Maria-Florina Balcan and Steve Hanneke}
\ShortHeadings{Robust Interactive Learning}{Balcan and Hanneke}
\firstpageno{1}

\begin{document}

\title{Robust Interactive Learning}

\author{\name Maria Florina Balcan \email ninamf@cc.gatech.edu \\
\addr 
School of Computer Science \\
Georgia Institute of Technology 
\\ \name Steve Hanneke \email shanneke@stat.cmu.edu \\
\addr 
Department of Statistics \\
Carnegie Mellon University 
}
\editor{someone}

\maketitle

\begin{abstract}
In this paper we propose and study a generalization of the standard
  active-learning model where a more general type of query, class conditional query, is
  allowed. Such queries have been quite useful in applications, 
but have been lacking theoretical understanding.
  In this work, we  characterize the power of such queries under two well-known noise models. We give nearly tight upper and lower bounds on the number of queries needed to learn both for the general agnostic setting and for
the bounded noise model.
  We further show that our methods can be made adaptive to the (unknown) noise rate,
with only negligible loss in query complexity.
\end{abstract}

\section{Introduction}
\label{sec:introgeneralinteractive}
The ever-expanding range of application areas for machine learning,
together with huge increases in the volume of raw data available, has
encouraged researchers to look beyond the classic paradigm of passive
learning from labeled data only.  Perhaps the most extensively used
and studied technique in this context is Active Learning, where the
algorithm is presented with a large pool of unlabeled examples (such
as all images available on the web) and can interactively ask for the
labels of examples of its own choosing from the pool.  The aim is to use
this interaction to drastically reduce the number of labels needed
(which are often the most expensive part of the data collection
process) in order to reach a low-error hypothesis.

Over the past ten years there has been a great deal of progress on
understanding active learning and its underlying
principles~\citep*{BBL,BBZ07,BDL09,CN07,dhsm,Hanneke07,Nina08,hanneke:thesis,Kol10,nips09,nips10}.
However, while useful in many applications~\citep*{McNi98,TonKol01},
requesting the labels of select examples
is only one very specific type of interaction
between the learning algorithm and the labeler.   When
analyzing many real world situations, it is desirable to consider
learning algorithms that make use of other types of queries as
well. For example, suppose we are actively learning a multiclass image
classifier from examples. If at some point, the algorithm needs an
image from one of the classes, say an example of ``house'', then an
algorithm that can only make individual label requests may need to ask
the expert to label a large number of unlabeled examples before it
finally finds an example of a house for the expert to label as
such. This problem could be averted by simply allowing the algorithm
to display a list of around a hundred thumbnail images on the screen,
and ask the expert to point to an image of a house if there is
one. The expert can visually scan through those images looking for a
house much more quickly than she can label every one of them. So in
this case, we get a significant increase in power by being able to ask
a particular type of query. In fact, queries of this type have been quite useful in several applications~\citep*{kollercquery2005,doyle:09},
but unfortunately, they have been lacking a principled theoretical understanding.

  In this work we  expand the study of active
learning by considering a model that allows us to analyze queries
motivated by such applications.  Specifically, the query protocol we analyze, namely class-conditional queries,
is based on the ability to ask for an example of a given label within a given set of unlabeled examples.
That is, the algorithm is provided with a large pool of unlabeled examples,
and may interact with an oracle as follows.
In each query, the algorithm proposes a label and a subset of the unlabeled examples,
and asks the oracle to point to one of these examples whose true label agrees with the specified label, if any exist.
This is a strict generalization of the traditional model of active learning by label requests.

It is well known that if the target function resides in a known concept class
and there is no classification noise (the so-called \emph{realizable case}),
then a simple approach based on the Halving algorithm~\citep*{litt} can learn a function
$\epsilon$-close to the target function using a number of queries dramatically
smaller than the number of random labeled examples required for PAC learning \citep*{hanneke:thesis}.

Encouraged by such strong results for the realizable case, we may
wonder whether equally strong reductions in query complexity are feasible in the
presence of classification noise.  In the present work, we find that
in the general agnostic case, this is not true when the noise rate is large, though a different type of reduction  is
consistently possible: namely, reduction by a factor related
to the overall noisiness of the data.  While this reduction
is much more modest than those achievable in the realizable case,
the fact that it is consistently available is interesting, in that
it contrasts with active learning, where the known improvements over passive
learning vary depending on the structure of the concept space \citep*{Hanneke07,hanneke:07b,dhsm}.
We also prove a sometimes stronger result in the special case of
bounded noise: namely, that compared to \emph{active learning},
the query complexity with class conditional queries is reduced by
a factor related to the noise bound.

\paragraph{Our Results}  
We provide the first general results concerning the query complexity of class-conditional queries in the presence of noise in a multiclass setting.
In particular:
\begin{list}{\labelitemi}{\leftmargin=1.5em}
\item[1.] In the purely agnostic case with noise rate $\eta$, we show that any interactive learning algorithm in this model seeking a classifier of error at most
$\eta+\epsilon$ must make $\Omega(d\eta^2 / \epsilon^2)$ queries, where $d$ is the Natarajan dimension;
we also provide a nearly matching upper bound of $\tilde{O}(d \eta^2 / \epsilon^2)$, for a constant number of classes.
This is smaller by a factor of $\eta$ compared to the sample complexity of passive learning,
and represents a reduction over the known results for the query complexity of active learning in many cases.
\item[2.] In the bounded noise model,
we provide nearly tight upper and lower bounds on the query complexity of the general query model as a function of the query complexity of active learning.
In particular, we find that the query complexity of the general query model is essentially reduced by a factor of the noise bound, compared to active learning.
\item[3.] We further show that our methods can be made adaptive to the (unknown) noise rate $\eta$,
with only negligible loss in query complexity.
\end{list}

Overall, we find that the reductions in query complexity for this model, compared to the traditional active learning model, are
largely concerned with a factor relating to the noise rate of the learning problem, so that the closer to the realizable case we are, the
greater the potential gains in query complexity.  However, for larger noise rates, the benefits are more modest,
a fact that sharply contrasts with the enormous benefits of using these types of queries in the realizable case;
this is true even for very benign types of noise, such as bounded noise, a fact that may seem surprising, especially since
the query complexity of the traditional active learning model is essentially unchanged (up to constant and log factors) by the presence of bounded noise, compared to the realizable case \citep*{kaariainen:06}.
We hope our analysis will help inform the use of these queries in practical learning problems, as well as provide a point of reference for future exploration of the general
topic of interactive machine learning.

\section{Formal Setting}
\label{sec:model}

We consider an interactive learning setting defined as follows.
There is an \emph{instance space} $\mathcal{X}$, a \emph{label space} $\Y$, and
some fixed \emph{target distribution} $\D_{XY}$ over $\mathcal{X} \times \Y$, with marginal $\D_{X}$ over $\X$.
Focusing on multiclass classification, we assume that $\Y = \{1,2,\ldots,k\}$, for some $k \in \mathbb{N}$.
In the learning problem, there is an i.i.d. sequence of random variables $(x_1,y_1), (x_2,y_2), (x_3,y_3), \ldots $, each with distribution $\D_{XY}$.
The learning algorithm is permitted direct access to the sequence of $x_i$ values (unlabeled data points).
However, information about the $y_i$ values is obtainable only via interaction with an oracle, defined as follows.

At any time, the learning algorithm may propose a label $\ell \in \Y$
and a finite subsequence of unlabeled examples $S=\{x_{i_1}, ..., x_{i_m}\}$ (for any $m \in \nats$);
if $y_{i_j} \neq \ell$ for all $j \leq m$, the oracle returns ``none.''  Otherwise,
the oracle selects an arbitrary $x_{i_j} \in S$ for which $y_{i_j}=\ell$ and returns the pair $(x_{i_j},y_{i_j})$.
In the following we call this model the $\CCQM$ (class-conditional queries) interactive learning model.
Technically, we implicitly suppose the set $S$ also specifies the unique indices of the examples it contains,
so that the oracle knows which $y_i$ corresponds to which $x_{i_j}$ in the sample $S$; however,
we make this detail implicit below to simplify the presentation.

In the analysis below, we fix a set of
classifiers $h : \X \to \Y$ called the \emph{hypothesis class}, denoted $\C$.
We will denote by $d$ the Natarajan dimension of $\C$ \citep*{natarajan:89,haussler:95,ben-david:95},
defined as the largest $m \in \nats$ such that $\exists (a_1,b_1,c_1),\ldots,(a_m,b_m,c_m) \in \X \times \Y \times \Y$
such that $\{b_1,c_1\} \times \cdots \times \{b_m,c_m\} \subseteq \{(h(a_1),\ldots,h(a_m)) : h \in \C\}$.
The Natarajan dimension has been calculated for a variety of hypothesis classes, and is known 
to be related to several other commonly used dimensions, including the pseudo-dimension 
and graph dimension \citep*{haussler:95,ben-david:95}.  For instance, for neural networks of $n$ nodes
with weights given by $b$-bit integers, the Natarajan dimension is at most $bn(n-1)$ \citep*{natarajan:89}.

For any $h :\X \to \Y$ and distribution $P$ over $\X \times \Y$,
define the \emph{error rate} of $h$ as $\err_{P}(h) = \P_{(X,Y)\sim P}\{h(X) \neq Y\}$; when $P = \D_{XY}$,
we abbreviate this as $\err(h)$.
For any finite sequence of labeled examples $L = \{(x_{i_1},y_{i_1}), \ldots, (x_{i_m},y_{i_m})\}$,
we define the empirical error rate $\err_{L}(h) = |L|^{-1} \sum_{(x,y) \in L} \mathbb{I}[h(x) \neq y]$.
In some contexts, we also refer to the empirical error rate on a finite sequence of \emph{unlabeled}
examples $U = \{x_{i_1}, \ldots, x_{i_m}\}$, in which case we simply define
$\err_{U}(h) = |U|^{-1} \sum_{x_{i_j} \in U} \mathbb{I}[h(x_{i_j}) \neq y_{i_j}]$, where the $y_{i_j}$ values
are the actual labels of these examples.

Let $h^*$ be the classifier in $\C$ of smallest $\err(h^*)$ (for simplicity, we suppose the minimum is always realized),
and let $\eta = \err(h^*)$, called the \emph{noise rate}.
The objective of the learning algorithm is to identify some $h$ with $\err(h)$
close to $\eta$ using only a small number of queries.
In this context, a \emph{learning algorithm} is simply any algorithm that makes some number of
queries and then halts and returns a classifier.  We are particularly interested in the following quantity.

\begin{definition}
\label{defn:query-complexity}
For any $\epsilon,\delta \in (0,1)$, any hypothesis class $\C$, and any family of distributions $\Dset$ on
$\X \times \Y$, define the quantity $\SC_{\CCQM}(\epsilon, \delta, \C,\Dset)$
as the minimum $q \in \nats$ such that there exists a learning algorithm $\alg$, which
for any target distribution $\D_{XY} \in \Dset$, with probability at least $1-\delta$,
makes at most $q$ queries and then returns a classifier $\hat{h}$ with $\err(\hat{h}) \leq \eta + \epsilon$.
We generally refer to the function $\SC_{\CCQM}(\cdot,\cdot,\C,\Dset)$
as the \emph{query complexity} of learning $\C$ under $\Dset$.
\end{definition}

The query complexity, as defined above, represents a kind of minimax statstical analysis,
where we fix a family of possible target distributions $\Dset$, and calculate, for the best
possible learning algorithm, how many queries it makes under its worst possible target
distribution $\D_{XY}$ in $\Dset$. Specific families of target distributions we will be interested
in include the random classification noise model, the bounded noise model, and the agnostic model which we define formally in the
corresponding sections.
In some contexts, we may also discuss the query complexity achieved by a particular algorithm,
in which case it is merely the same definition as above except replacing $\alg$ with the particular
algorithm in question.

\section{The General Agnostic Case}
\label{general-queries}

We start by considering the most general, \emph{agnostic} setting,
where we consider arbitrary noise distributions subject to a constraint on the noise rate.
This is particularly relevant to many practical scenarios, where we often do not know what
type of noise we are faced with, potentially including stochastic labels or model misspecification,
and we would therefore like to refrain from making any specific assumptions about the
nature of the noise.
Formally, the family of distributions we consider is
$\Agnostic(\C,\alpha) = \{ \D_{XY} : \inf_{h \in \C} \err(h) \leq \alpha\}$, $\alpha \in [0,1/2)$.
In this section we prove nearly tight upper and lower bounds on the query complexity of our model.
Specifically, supposing $k$ is constant, we have the following theorem.  

\begin{theorem}
\label{thm:main-agnostic}
For any hypothesis class $\C$ of Natarajan dimension $d$, for any $\eta \in [0,1/32)$,
\[\SC_{\CCQM}(\epsilon,\delta,\C,\Agnostic(\C,\eta)) = \tilde{\Theta}\left( d \frac{\eta^2}{\epsilon^2}\right).\]
\end{theorem}

The first interesting thing is that our bound differs from the sample complexity of passive learning only in a factor of $\eta$.
This contrasts with the realizable case, where it is possible to learn with a query complexity that is exponential smaller than the query complexity of passive learning. On the other hand, is also interesting that this factor of $\eta$ is consistently available regardless of the structure of the concept space. This contrasts with active learning where the  extra factor of $\eta$ is only available in certain special cases~\citep*{Hanneke07}.

\subsection{Proof of the Lower Bound}
\label{lowrboundgeneralnoise}

We first prove the lower bound.  We specifically prove that
for $0 < 2\epsilon \leq \eta < 1/4$,
\[\SC_{\CCQM}(\epsilon,1/4,\C,\Agnostic(\C,\eta)) = \Omega\left(d \frac{\eta^2}{\epsilon^2}\right).\]
Monotonicity in $\delta$ extends this to any $\delta \in (0,1/4]$.

\begin{proof}
The key idea of the proof is to provide a reduction from the (binary) active learning model (label request queries)
to our multiclass interactive learning model (general class-conditional queries) for the hard case known previously
in the literature for the active learning model~\citep*{BDL09}.

In particular, consider a set of $d$ points $x_0$,  $x_1$, $x_2$,..., $x_{d-1}$ shattered
by $\C$, and let $(y_0,z_0),$ $\ldots,$ $(y_{d-1},z_{d-1})$ be the label pairs that witness the shattering.
Here is a distribution over $\X \times \Y$ : point
$x_0$ has probability $1-\beta$, while each of the remaining
$x_i$ has probability $\beta/(d-1)$, where $\beta=2(\eta+2\epsilon)$. At
$x_0$ the response is always $Y=y_0$. At $x_i$, $ 1\leq i \leq d-1$, the
response is $Y = z_i$ with probability $1/2 +  \gamma b_i$ and $Y = y_i$ with probability $1/2 - \gamma b_i$,
where $b_i$ is either $+1$ or $-1$, and $\gamma=2\epsilon/\beta=\epsilon/(\eta+2\epsilon).$

\citet*{BDL09} show that for any active learning algorithm, one can set
$b_0=1$ and all the $b_i$, $i \in \{1, \ldots, d-1\}$ in a certain way
so that the algorithm must make $\Omega(d\eta^2 / \epsilon^2)$ queries
in order to output a classifier of error at most $\eta+\epsilon$ with
probability at least $1/2$.  Building on this, we can show any
interactive learning algorithm seeking a classifier of error at most
$\eta+\epsilon$ must make $\Omega(d\eta^2 / \epsilon^2)$ queries to
succeed with probability at least $1/2$.

Assume that we have an algorithm $\alg$ that works
for the $\CCQM$ model with query complexity
$\SC_{\CCQM}(\epsilon,\delta,\C,\Agnostic(\C,\eta))$. We show how to
use $\alg$ as a subroutine in an active learning algorithm that is
specifically tailored to the above hard set of distributions.

In particular, we can simulate an oracle for the $\CCQM$ algorithm as follows.
Suppose our $\CCQM$ algorithm queries with a set $S_i$ for a label $\ell$.
If $\ell$ is not one of the $y_0,\ldots,y_{d-1},z_0,\ldots,z_{d-1}$ labels, we may immediately return that none exist.
If there exists $x_{i,j} \in S_i$ such
that $x_{i,j}=x_0$ and $\ell = z_0$, then we may simply return to
the algorithm this $(x_{i,j},z_0)$.
Otherwise, we need only make (in expectation)
$\frac{1}{1/2-\gamma}$ active learning queries to respond to the class-conditional query, as follows.
We consider the subset $R_i$ of $S_i$ of points $x_{i,j}$ among those $x_j$ with $\ell \in \{y_j,z_j\}$.
We pick an example
$x_{i}^{(1)}$ at random in $R_i$ and request its label $y_{i}^{(1)}$.
If $x_{i}^{(1)}$ has label $y_{i}^{(1)} = \ell$, then we return to the algorithm $(x_{i}^{(1)},y_{i}^{(1)})$;
otherwise, we continue sampling random $x_{i}^{(2)},x_{i}^{(3)},\ldots$ points from $R_i$ (whose labels have not yet been requested)
and requesting their labels $y_{i}^{(2)},y_{i}^{(3)},\ldots$, until we find one with label $\ell$, at which point we return to the algorithm that example.
If we exhaust $R_i$ without finding such an example, we return to the algorithm that no such point exists.
Since each $x_{i,j}\in R_i$ has probability at least $1/2-\gamma$ of having $y_{i,j} = \ell$,
we can answer any query of $\alg$ using in expectation no more than $\frac{1}{1/2-\gamma}$ label request
queries.

In particular, we can upper bound this number of queries by a geometric random variable and apply concentration inequalities for geometric random variables to bound the total number of label requests, as follows.
Let $A_i$ be a random variable indicating the actual number of label requests we make to answer query number $i$
in the reduction above, before returning a response. We can show that
For $j \leq A_i$, if $h^*(x_{i}^{(j)}) \neq \ell$, let $Z_j = I[y_{i}^{(j)} = \ell]$,
and if $h^*(x_{i}^{(j)}) = \ell$, let $C_j$ be an independent Bernoulli($(1/2-\gamma)/(1/2+\gamma)$) random variable,
and let $Z_j = C_j I[y_{i}^{(j)}=\ell]$.
For $j > A_i$, let $Z_j$ be an independent Bernoulli($1/2-\gamma$) random variable.
Let $B_i = \min\{ j : Z_j = 1\}$.
Since, $\forall j \leq A_i$, $Z_j \leq I[y_{i}^{(j)}=\ell]$, we clearly have $B_i \geq A_i$.
Furthermore, note that the $Z_j$ are independent Bernoulli($1/2-\gamma$) random variables,
so that $B_i$ is a Geometric($1/2-\gamma$) random variable.
By  Lemma~\ref{coins} in Appendix~\ref{useful},
we obtain that with probability at least $3/4$ we have
that if $Q$ is any constant and $\alg$ makes $\leq Q$ queries, then with probability at least $3/4$,
$\sum_i A_i \leq \sum_{i=1}^{Q} B_i \leq \frac{2}{1/2-\gamma} [Q + 4 \ln(4)]$.
Thus, since $\sum_i A_i$ represents the total number of label requests made by this algorithm,
and we know that with probability at least $3/4$ the number of queries is at most
$Q = \SC_{\CCQM}(\epsilon,1/4,\C,\Agnostic(\C,\eta))$,
combining this together with the aforementioned~\citep*{BDL09} lower bound for active learning, we obtain the result.
\end{proof}

\subsection{Upper bound}
\label{known-eta}

In this section we describe an algorithm whose query complexity is $\tilde{O}\left(kd \frac{\beta^2}{\epsilon^2}\right)$.
  For clarity, we start by considering  in the case where we know an upper bound $\beta$ on $\eta$. This procedure (Algorithm~\ref{alg-noise-general}) has two phases:
in Phase~\ref{phaseone}, it uses a robust version of the classic halving algorithm to produce a classifier whose error rate is at most $10(\beta+\epsilon)$ by only using $\tilde{O}\left(kd \log\frac{1}{\epsilon}\right)$  queries.  In Phase~\ref{phasetwo}, we run a refining algorithm that uses $\tilde{O}\left( kd \frac{\beta^2}{\epsilon^2}\right)$ queries to turn the classifier output in phase one into a classifier of error $\eta+\epsilon$.  We will discuss how to remove the assumption of knowing an upper bound $\beta$ on $\eta$, adapting to $\eta$,
 in Section~\ref{sec:unknown-rate}.

\begin{algorithm}
{\small
{\bf Input}: The sequence $(x_1, x_2, ..., ) $;  
values $u$, $\s$, $\delta$;
budget $n$ (optional; default value $= \infty$).
\begin{list}{\labelitemi}{\leftmargin=1.5em}
\item[1.] Let $V$ be a (minimal) $\epsilon$-cover of the space of classifiers $\C$ with respect to  $\D_{X}$.
Let $U$ be $\{x_1, ..., x_u\}$.
\item[2.] Run the Generalized~Halving~Algorithm (Phase~\ref{phaseone}) with input $U$; $V$, $\s$, $c \ln \frac{4 \log_{2} |V|}{\delta}$, $n/2$, and get $h$ returned.
\item[3.] Run the Refining Algorithm (Phase~\ref{phasetwo}) with input $U$, $h$, $n/2$, and get labeled sample $L$ returned.
\item[4.] Find a hypothesis $h' \in V$ of minimum $\err_{L}(h')$.
\end{list}
{\bf Output} Hypothesis $h'$ (and $L$).
}
\caption{{\small General Agnostic Interactive Algorithm}}
\label{alg-noise-general}
\end{algorithm}

Before presenting and analyzing the main steps of our algorithm, we start by describing a useful definition and a useful subroutine (Subroutine~\ref{find-mistake}, Find-Mistake).
Given $V \subseteq \C$, we define the plurality vote classifier as
$$\plur(V)(x)=\argmax_{y \in \Y}(\sum_{h \in V}(\mathbb{I} [h(x) =y]).$$
\begin{subroutine}
{\small
{\bf Input}: The sequence $S = (x_1,x_2,\ldots,x_{m})$; classifier $h$
\begin{list}{\labelitemi}{\leftmargin=1.5em}
\item[1.] For each $y \in \{1,\ldots,k\}$,
\begin{itemize}
\item[(a)] Query the set $\{x \in S : h(x) \neq y\}$ for label $y$
\item[(b)] If received back an example $(x,y)$, return $(x,y)$
\end{itemize}
\item[2.] Return ``none''
\end{list}
}
\caption{{\small Find-Mistake}}
\label{find-mistake}
\end{subroutine}

Note that, if $\err_{S}(h) > 0$, then Find-Mistake returns a labeled example $(x,y)$ with $y$ the
true label of $x$, such that $h(x) \neq y$, and otherwise it returns an indication that no such point exists.

\begin{phase}
{\small
{\bf Input}: The sequence $U=(x_1, x_2, ...,x_{\ps}) $; set of classifiers $V$; values $\s$, $N$; budget $n$ ($n$ optional: default value $= \infty$).
\begin{list}{\labelitemi}{\leftmargin=1.5em}
\item[1.] Set $b=\true$, $t = 0$.
\item[2.] while ($b$ and $t \leq n - N$)
\begin{itemize}
\item[(a)] Draw $S_1$, $S_2$, ..., $S_N$ of size $\s$ uniformly without replacement from $U$.
\item[(b)] For each $i$, call Find-Mistake with arguments $S_i$, and $\plur(V)$.  If it returns a mistake, we record the mistake $(\tx_i,\ty_i)$ it returns.
\item[(c)] If Find-Mistake finds a mistake in more than $N/3$ of the sets, remove from $V$ every $h\in V$ making mistakes on $>N/9$ examples $(\tx_i,\ty_i)$, and set $t \gets t+N$; else $b\gets 0$.
\end{itemize}
\end{list}
{\bf Output} Hypothesis $\plur(V)$.
}\caption{{\small Generalized Halving Algorithm}}
\label{phaseone}
\label{generalized-halving-algorithm}
\end{phase}

\begin{phase}
{\small
{\bf Input}: The sequence $U=(x_1, x_2, ..., x_{\ps}) $; classifier $h$; budget $n$ ($n$ optional: default value $= \infty$).
\begin{list}{\labelitemi}{\leftmargin=1.5em}
\item[1.]  Set $b=1$, $t=0$, $W=U$, $L = \emptyset$.
\item[2.] while ($b$ and $t < n$)
\begin{itemize}
\item[(a)] Call Find-Mistake with arguments $W$, and $h$.
\item[(b)] If it returns a mistake $(\tx,\ty)$, then set $L \gets L \cup \{(\tx,\ty)\}$, $W\gets W\setminus \{\tx\}$, and $t  \gets t+1$.
\item[(c)] Else set $b=0$ and $L \gets L \cup \{(x,h(x)) : x \in W\}$.
\end{itemize}
\end{list}
{\bf Output} Labeled sample $L$.
}\caption{{\small Refining Algorithm}}
\label{phasetwo}
\label{refining-algorithm}
\end{phase}

Lemma~\ref{lem:generalized-halving-algorithm} below characterizes the performance of Phase~\ref{phaseone} and Lemma~\ref{lem:good-labels} characterizes the performance of  Phase~\ref{phasetwo}.
Note that the budget parameter in these methods is only utilized in our later discussion of adaptation to the noise rate.

\begin{lemma}
\label{lem:generalized-halving-algorithm} Assume that some $\hat{h} \in V$ has $\err_U(\hat{h}) \leq \beta$ for $\beta \in [0,1/32]$.
With probability $\geq 1-\delta/2$, running Phase~\ref{generalized-halving-algorithm}
with $U$, and values $\s=\left\lfloor\frac{1}{16 \beta}\right\rfloor$ and $N = c \ln \frac{4\log_{2} |V|}{\delta}$ (for an appropriate constant $c \in (0,\infty)$),
we have that for every round of the loop of Step 2, the following hold.
\begin{list}{\labelitemi}{\leftmargin=0.5em}
\item $\hat{h}$ makes mistakes on at most $N/9$ of the returned $(\tx_i,\ty_i)$ examples.
\item If $\err_U(\plur(V)) \geq 10 \beta$, then Find-Mistake returns a mistake for $\plur(V)$ on $> N/3$ of the sets.
\item If Find-Mistake returns a mistake for $\plur(V)$ on $> N/3$ of the sets $S_i$, then the number of $h$ in $V$ making mistakes on $> N/9$ of the returned $(\tx_i,\ty_i)$ examples in Step $3(b)$ is at least $(1/4) |V|$.
\end{list}
\end{lemma}
\begin{proof}
Phase~\ref{generalized-halving-algorithm} and Lemma~\ref{lem:generalized-halving-algorithm} are inspired by the analysis of \citet*{hanneke:07b}.
In the following, by a \emph{noisy} example we mean any $x_i$ such that $\hat{h}(x_i) \neq y_i$.
 The expected number of noisy points in any given set $S_i$ is at most $1/16$, which (by Markov's inequality)
 implies the probability $S_i$ contains a noisy point is at most $1/16$.  Therefore, the expected number
 of sets $S_i$ with a noisy point in them is at most $N / 16$, so
 by a Chernoff bound, with probability at least $1-\delta/(4 \log_{2} |V|)$ we have that at most
$N/9$ sets $S_i$ contain any noisy point, establishing claim 1.

 Assume that $\err_U(\plur(V)) \geq 10 \beta$. The probability that there is a point $\tx_i$ in $S_i$ such that $\plur(V)$ labels $\tx_i$
 differently from $\ty_i$ is
 $\geq 1-(1-10\beta)^{\s} \geq .37$ (discovered by direct optimization).
 So (for an appropriate value of $c > 0$ in $N$) by a Chernoff bound, with probability at least $1-\delta/(4 \log_{2} |V|)$, at least
 $N/3$ of the sets $S_i$ contain a point $\tx_i$ such that $\plur(V)(\tx_i) \neq \ty_i$, which establishes claim 2.
 Via a combinatorial argument, this then implies with probability at least $1-\delta/(4 \log_2 |V|)$,
 at least $|V|/4$ of the hypotheses will make mistakes on more than $N/9$ of the sets $S_i$.
  To see this consider the bipartite graph
  where on the left hand side we have all the classifiers in $V$ and on the right hand side we have all the returned $(\tx_i,\ty_i)$ examples.
  Let us put an edge between a node $i$ on the left and a node $j$ on the right if the hypothesis $h_i$ associated to node $i$ makes a mistake
  on $(\tx_i,\ty_i)$. Let $M$ be the number of vertices in the right hand side.
  Clearly, the total number of edges in the graph is at least $(1/2) |V| |M|$, since at most $|V|/2$ classifiers label $\tx_i$ as $\ty_i$.
  Let $\alpha |V|$ be the number of classifiers in $V$ that make mistakes
  on at most $N/9$ $(\tx_i,\ty_i)$ examples. The total number of edges in the graph is then upper bounded by
 $\alpha |V| N/9 + (1-\alpha) |V| M.$
Therefore,  $$(1/2) |V| |M| \leq \alpha |V| N/9 + (1-\alpha) |V| M,$$
which implies
  $$ |V| |M| (\alpha-1/2) \leq \alpha |V| N/9 .$$
  Applying the lower bound $M \geq N/3$,
  we get  $(N/3) |V| (\alpha-1/2) \leq \alpha |V| N/9$, so
  $\alpha \leq 3/4$.  This establishes claim 3.

A union bound over the above two events, as well as over the iterations of the loop (of which there are at most $\log_{2} |V|$ due to the third claim of this lemma) obtains the claimed overall $1-\delta/2$ probability.
 \end{proof}

\begin{lemma}
\label{lem:good-labels}
Suppose some $\hat{h}$ has $\err_{U}(\hat{h}) \leq \beta$, for some $\beta \in [0,1/32]$.
Running Phase~\ref{refining-algorithm} with parameters $U$, $\hat{h}$, and any budget $n$,
if $L$ is the returned sample, and $|L| = |U|$, then every $(x_i,y) \in L$ has $y = y_i$ (i.e.,
the labels are in agreement with the oracle's labels); furthermore, $|L|=|U|$ definitely happens for any
$n \geq \beta |U| + 1$.
\end{lemma}
\begin{proof}
Every call to Find-Mistake returns a new mistake for $\hat{h}$ from $U$, except the last call, and
since there are only $\beta |U|$ such mistakes, the procedure requires only $\beta |U|+1$ calls to Find-Mistake.
Furthermore, every label was either given to us by the oracle, or was assigned at the end, and in this
latter case the oracle has certified that they are correct.

Formally, if $|L| = |U|$, then either every $x \in U$ was returned as some $(\tx,\ty)$ pair in Step 2.b,
or we reached Step 2.c.  In the former case, these $\ty$ labels are the oracle's actual responses,
and thus correspond to the true labels.
In the latter case, every element of $L$ added prior to reaching 2.c was returned by the oracle,
and is therefore the true label.
Every element $(x_i,y) \in L$ added in Step 2.c has label $\hat{h}(x_i)$, which the
oracle has just told us is correct in Find-Mistake (meaning we definitely have $\hat{h}(x_i) = y_i$).
Thus, in either case, the labels are in agreement with the true labels.
Finally, note that each call to Find-Mistake either returns a mistake for $\hat{h}$ we have not previously received,
or is the final such call.  Since there are at most $\beta |U|$ mistakes in total, we can have at most $\beta |U| + 1$ calls to Find-Mistake.
\end{proof}

We are now ready to present our main upper bounds for the agnostic noise model.

\begin{theorem}
\label{main-upper}
Suppose $\beta \geq \eta$, and $\beta + \epsilon \leq 1/32$.
Running Algorithm~\ref{alg-noise-general} with parameters
$u = O(d ((\beta + \epsilon)/\epsilon^2)\log(k/\epsilon\delta))$,
$\s=\left\lfloor\frac{1}{16 (\beta+\epsilon)}\right\rfloor$, and $\delta$,
with probability at least $1-\delta$ it produces a classifier $h^{\prime}$ with $\err(h^{\prime}) \leq \eta + \epsilon$
using a number of queries
$O\left(k d \frac{\beta^2}{\epsilon^2} \log \frac{1}{\epsilon \delta}  + k d \log \frac{\log(1/\epsilon)}{\delta} \log\frac{1}{\epsilon}\right)$.
\end{theorem}
\begin{proof}
We have chosen $u$ large enough so that $\err_{U}(h^*) \leq \eta + \epsilon \leq \beta + \epsilon$, with probability at least $1-\delta/4$,
by a (multiplicative) Chernoff bound.
By Lemma~\ref{lem:generalized-halving-algorithm}, we know that with probability $1-\delta/2$,
$h^*$ is never discarded in Step 2(c) in Phase~\ref{generalized-halving-algorithm}, and as long as $\err_U(\plur(V)) \geq 10 (\beta+\epsilon)$,
then we cut the set $|V|$ by a constant factor.  So, with probability $1-3\delta/4$, after at most $O(k N \log(|V|))$ queries,
Phase~\ref{generalized-halving-algorithm} halts with the guarantee that  $\err_U(\plur(V)) \leq  10  (\beta+\epsilon)$.
Thus, by Lemma~\ref{lem:good-labels}, the execution of Phase~\ref{refining-algorithm} returns a set $L$ with the true labels after at most
$(10 (\beta + \epsilon) u + 1)k$ queries.

Furthermore, we can choose the $\epsilon$-cover $V$ so that $|V| \leq 4 (c k^2 / \epsilon)^{d}$ for an appropriate constant $c$ \citep*{van-der-Vaart:96,haussler:95}.  

Therefore, by Chernoff and union bounds, we have chosen $u$ large enough so that the $h^{\prime}$ of minimal $\err_{U}(h^{\prime})$
has $\err(h^{\prime}) \leq \eta + \epsilon$ with probability at least $1-\delta/4$.
Combining the above events by a union bound, with probability $1-\delta$,
the $h^{\prime}$ chosen at the conclusion of Algorithm~\ref{alg-noise-general} has
$\err(h^{\prime}) \leq \eta + \epsilon$ and the total number of queries is at most
$$kN \log_{4/3}(|V|) + k(10 (\beta + \epsilon) u + 1) = O\left(kd \log \frac{d \log (k/\epsilon)}{\delta} \log\frac{1}{\epsilon} + kd \frac{(\beta + \epsilon)^2}{\epsilon^2} \log\frac{k}{\epsilon \delta}\right).$$
\end{proof}

In particular, if we take $\beta = \eta$, Theorem~\ref{main-upper} implies the upper bound part of Theorem~\ref{thm:main-agnostic}.

\noindent {\bf Note}: It is sometimes desirable to restrict the size of the sample we make the query for, so that
the oracle does not need to sort through an extremely large sample searching for a mistake.
To this end, we can run Phase~\ref{refining-algorithm} on chunks of size $1/(\eta+\epsilon)$ from $U$, and then union
the resulting labeled samples to form $L$.  The number of queries required for this is still bounded
by the desired quantity.

\label{sec:unknown-rate}

In practice, knowledge of an upper bound $\beta$ reasonably close to $\eta$ is typically not available.
As such, it is important to design algorithms that adapt to the unknown value of $\eta$ using only
observable quantities.  The following theorem indicates this is possible in our setting,
without significant loss in query complexity.

\begin{theorem}
\label{thm:adaptive-agnostic}
There exists an algorithm that is independent of $\eta$
and $\forall \eta \in [0,1/2)$ achieves query complexity
$\SC_{\CCQM}(\epsilon,\delta,\C,\Agnostic(\C,\alpha)) =
\tilde{O}\left(kd \frac{\eta^2}{\epsilon^2}\right)$.
\end{theorem}

\begin{proof}
We consider the proof of this theorem in two stages, with the following intuitive motivation.
First, note that if we set the budget parameter $n$ large enough
(at roughly $1/k$ times the value of the query complexity bound of Theorem~\ref{thm:main-agnostic}),
then the largest value of $\beta$ for which the algorithm (with parameters as in Theorem~\ref{main-upper})
produces $L$ with $|L| = u$ has $\beta \geq \eta$, so that it produces $h^{\prime}$ with $\err(h^{\prime}) \leq \eta + \epsilon$.
So for a given budget $n$, we can simply run the algorithm for each $\beta$ value in a log-scale grid
of $[\epsilon,1]$, and take the $h^{\prime}$ for the largest such $\beta$ with $|L| = u$.
The second part of the problem then becomes determining an appropriately large budget $n$,
so that this works.  For this, we can simply search for such a value by a guess-and-double
technique, where for each $n$ we check whether it is large enough by evaluating a standard
confidence bound on the excess error rate; the key that allows this to work is that, if $|L| = u$,
then the set $L$ is an i.i.d. $\D_{XY}$-distributed sequence of labeled examples, so that we
can use known confidence bounds for working with sequences of random labeled examples.
The details of this strategy follow.

Consider values $n_j = 2^{j}$ for $j \in \nats$, and define the following procedure.
We can consider a sequence of values $\eta_i = 2^{1-i}$ for $i \leq \log_{2}(1/\epsilon)$.
For each $i = 1,2, \ldots,\log_{2}(1/\epsilon)$, we run Algorithm~\ref{alg-noise-general}
with parameters $$u = u_i = O(d((\eta_i + \epsilon) / \epsilon^2) \log(k /\epsilon \delta)),$$ 
$$\s = \s_i = \frac{1}{16 (\eta_i + \epsilon)},~~~\delta_i = \delta / (8 \log_{2}(1/\epsilon))$$
and budget parameter $n_{j} / \log_{2}(1/\epsilon)$.  Let $h_{j i}$ and $L_{j i}$ denote the return values
from this execution of Algorithm~\ref{alg-noise-general}, and let $\hat{h}_{j}$ and $\hat{L}_{j}$ denote the values $h_{j i}$ and $L_{j i}$,
respectively, for the smallest value of $i$ for which $|L_{j i}| = u_i$: that is, for which
the execution of Phase~\ref{refining-algorithm} ran to completion.

Note that for some $j$ with
$n_j = O\left(d \frac{\eta^2}{\epsilon^2} \log \frac{k\log_{2}(1/\epsilon)}{\epsilon \delta}  + d \log \frac{\log^2(1/\epsilon)}{\delta} \log\frac{k}{\epsilon}\right)\log_{2}\frac{1}{\epsilon}$,
Theorem~\ref{main-upper} implies that with probability $1-\delta/4$,
every $i \leq \lfloor \log_{2}(1 / \eta) \rfloor$ with $|L_{j i}| = u_i$ has $\err(h_{j i}) \leq \eta + \epsilon/2$,
and $|L_{j i}| = u_i$ for at least one such $i$ value: namely, $i = \lfloor \log_{2}(1 / \max\{\eta,\epsilon\}) \rfloor$.
Thus, $\err(\hat{h}_{j}) \leq \eta + \epsilon/2$ for this value of $j$.
Let $j^*$ denote this value of $j$, and for the remainder of this subsection we suppose this high-probability event occurs.

All that remains is to design a procedure for searching over $n_j$ values to find one large enough to obtain this error rate guarantee,
but not so large as to lose the query complexity guarantee.
Toward this end, define
\[\Epsilon_{j} = \frac{8 d}{|\hat{L}_{j}|} \ln\left(\frac{12 |\hat{L}_{j}| j^2}{\delta}\right) + \sqrt{\err_{\hat{L}_{j}}(\hat{h}_{j}) \frac{16 d}{|\hat{L}_{j}|} \ln\left(\frac{12 |\hat{L}_{j}| j^2}{\delta}\right)}.\] 
A result of \citet*{Vapnik:book98} (except substituting the appropriate quantities for the multiclass case)
implies that with probability at least $1-\delta/2$,
\[\forall j, \left|\left(\err_{\hat{L}_{j}}(\hat{h}_{j}) - \min_{h \in \C} \err_{\hat{L}_{j}}(h)\right) - \left(\err(\hat{h}_{j}) - \err(h^*)\right)\right| \leq \Epsilon_j.\]

Consider running the above procedure for $j = 1,2,3,\ldots$ in increasing order
until we reach the first value of $j$ for which
\[\err_{\hat{L}_{j}}(\hat{h}_{j}) - \min_{h \in \C} \err_{\hat{L}_{j}}(h) + \Epsilon_j \leq \epsilon.\]
Denote this first value of $j$ as $\hat{j}$.
Note that choosing $\hat{j}$ in this way guarantees $\err(\hat{h}_{\hat{j}}) \leq \eta + \epsilon$.

It remains only to bound the value of this $\hat{j}$, so that we may add up the total number of queries among the executions of our procedure for all values $j \leq \hat{j}$.
By setting the constants in $u_i$ appropriately, the sample size of $|\hat{L}_{j}|$ is large enough so that,
for $j = j^*$, a Chernoff bound (to bound $\err_{\hat{L}_{j}}(h^*) \geq \err_{\hat{L}_{j}}(\hat{h}_j)$) 
guarantees that with probability $1-\delta/4$, $\Epsilon_j \leq \epsilon / 4$.
Furthermore, we have
\[\err_{\hat{L}_{j}}(\hat{h}_{j}) - \min_{h \in \C} \err_{\hat{L}_{j}}(h) \leq \err(\hat{h}_{j}) - \err(h^*) + \Epsilon_j \leq \epsilon / 2 + \epsilon / 4 = (3/4) \epsilon,\]
so that in total $\err_{\hat{L}_{j}}(\hat{h}_{j}) - \min_{h \in \C} \err_{\hat{L}_{j}}(h) + \Epsilon_j \leq (3/4) \epsilon + \epsilon / 4 = \epsilon$.
Thus, we have $\hat{j} \leq j^*$, so that the total number of queries is less than $2 n_{j^*}$.

Therefore, by a union bound over the above events, with probability $1-\delta$,
the selected $\hat{h}_{\hat{j}}$ has $\err(\hat{h}_{\hat{j}}) \leq \eta + \epsilon$, and
the total number of queries is less than
\[2 k n_{j^*} = O\left(d k \frac{\eta^2}{\epsilon^2} \log \frac{\log(1/\epsilon)}{\epsilon \delta}\log\frac{1}{\epsilon}  + d k \log \frac{\log(1/\epsilon)}{\delta} \log^2\frac{1}{\epsilon}\right).\]
Thus, not having direct access to the noise rate only increases our query complexity
by at most a logarithmic factor compared to the bound of Theorem~\ref{thm:main-agnostic}.
\end{proof}

\section{Bounded Noise}
\label{bounded noise}

In this section we study the \emph{Bounded noise} model (also known as Massart noise),
which has been extensively studied in the statistical learning theory literature \citep*{massart:06,gine:06,hanneke:11}. 
This model represents a significantly stronger restriction on the type of noise.
The motivation for bounded noise is that, in some scenarios, we do have an
accurate representation of the target function within our hypothesis class (i.e., the model is correctly specified),
but we allow for nature's labels to be slightly randomized.
Formally, the family of distributions we consider is
$\BoundedNoise(\C,\alpha) = \{ \D_{XY} : \exists h^* \in \C \text{ s.t. } \P_{\D_{XY}}( Y \neq h^*(X) | X) \leq \alpha\}$,  for $\alpha \in [0,1/2)$.
In some cases, we are interested in the special case of Random Classification Noise, defined as
$\RCN(\C,\alpha) = \{ \D_{XY} : \exists h^* \in \C \text{ s.t. } \forall \ell \neq h^*(x), \P_{\D_{XY}}( Y = \ell | X = x) = \alpha/(k-1)\}$.
We will also discuss $\BoundedNoise(\C,\alpha;\D_{X})$ and $\RCN(\C,\alpha;\D_{X})$ as those 
$\D_{XY}$ in these respective classes having marginal $\D_{X}$ on $\X$.

In this section we show a lower bound on the query complexity of
interactive learning with class-conditional queries as a function of the query
complexity of active learning (label request queries).
The proof follows via a reduction from the (multiclass) active learning model
(label request queries) to our interactive learning model
(general class-conditional queries), very similar in spirit to the
reduction given in the proof of the lower bound in
Theorem~\ref{thm:main-agnostic}.

\begin{theorem}
\label{lower-bound-bounded-noise}
Consider any hypothesis class $\C$ of Natarajan dimension
$d \in (0,\infty)$.  For any $\alpha \in [0,1/2)$, and any distribution $\D_{X}$ over $\X$,
in the random classification noise model we have the
following relationship between the query complexity of interactive
learning in the class-conditional queries model and the the query complexity of active learning with label requests:
$$\frac{\alpha}{2(k-1)} \SC_{\AL}(\epsilon,2\delta,\C,\RCN(\C,\alpha; \D_{X}))- 4\ln{\frac{1}{\delta}} \leq  \SC_{\CCQM}(\epsilon,\delta,\C, \RCN(\C,\alpha ; \D_{X}))$$
\end{theorem}
\begin{proof}
The proof follows via a reduction from the active learning model (label request queries) to our interactive learning model (general class-conditional queries).
Assume that we have an algorithm that works for the $\CCQM$ model with query complexity $\SC_{\CCQM}(\epsilon,\delta,\C,\RCN(\C,\alpha ; \D_{X}))$.
We can convert this into an algorithm that works in the active learning model with a query complexity of
$\SC_{\AL}(\epsilon,2\delta,\C,\RCN(\C,\alpha ; \D_{X})) =\frac{2 (k-1)}{\alpha} [\SC_{\CCQM}(\epsilon,\delta,\C,\RCN(\C,\alpha ; \D_{X})) + 4\ln{\frac{1}{\delta}}]$, as follows.
When our $\CCQM$ algorithm queries the $i^{\rm{th}}$ time, say querying for a label $y$ among a set $S_i$,
we pick an example $x_{i,1}$ at random in $S_i$ and (if the label of $x_{i,1}$ has never previously been requested),
we request its label $y_{i,1}$.
If $y = y_{i,1}$, then we return $(x_{i,1},y_{i,1})$ to the algorithm,
and otherwise we keep taking examples ($x_{i,2},x_{i,3},\ldots$) at random in the set $S_i$ and
(if their label has not yet been requested) requesting their labels ($y_{i,2},y_{i,3},\ldots$), until we find one with label $y$,
at which point we return this labeled example to the algorithm.
If we exhaust $S_i$ and we find example of label $y$, we return to the algorithm that there are no examples in $S_i$ with label $y$.

Let $A_i$ be a random variable indicating the actual number of label requests we make in round $i$
before getting either an example of label $y$ or exhausting the set $S_i$.
We also define a related random variable $B_i$ as follows.
For $j \leq A_i$, if $h^*(x_{i,j}) \neq y$, let $Z_j = I[y_{i,j} = y]$, and if $h^*(x_{i,j}) = y$, let $C_j$ be an independent Bernoulli($(\alpha/(k-1))/(1-\alpha)$) random variable, and let $Z_j = C_j I[y_{i,j}=y]$.
For $j > A_i$, let $Z_j$ be an independent Bernoulli($\alpha/(k-1)$) random variable.
Let $B_i = \min\{ j : Z_j = 1\}$.
Since, $\forall j \leq A_i$, $Z_j \leq I[y_{i,j}=y]$, we clearly have $B_i \geq A_i$.
Furthermore, note that the $Z_j$ are independent Bernoulli($\alpha/(k-1)$) random variables,
so that $B_i$ is a Geometric($\alpha/(k-1)$) random variable.
By  Lemma~\ref{coins} in Appendix~\ref{useful},
we obtain that with probability at least $1-\delta$ we have
$$\sum_i A_i \leq \sum_i B_i \leq \frac{2(k-1)}{\alpha} [\SC_{\CCQM}(\epsilon,\delta,\C,\RCN(\C,\alpha ; \D_{X})) + 4\ln{\frac{1}{\delta}}].$$
This then implies
$$\SC_{\AL}(\epsilon,2\delta,\C,\RCN(\C,\alpha ; \D_{X})) \leq  \frac{2(k-1)}{\alpha} [\SC_{\CCQM}(\epsilon,\delta,\C,\RCN(\C,\alpha ; \D_{X})) + 4\ln{\frac{1}{\delta}}],$$
which implies the desired result.
\end{proof}

To complement this lower bound, we prove a related upper bound via
an analysis of an algorithm below, which operates by reducing to a
kind of batch-based active learning algorithm.
Specifically, assume that we have an active learning algorithm $\A$ that operates as
follows. It proceeds in rounds and in each round it interacts with an
oracle by providing a region $R$ of the instance space and a number
$m$ and and it expects in return $m$ labeled examples from the
conditional distribution given that $x$ is in $R$.  For example the
$A^2$ algorithm~\citep*{BBL} and the algorithm of~\citet*{Kol10} can be
written to operate this way.  We show in the following how we can use
our algorithms from Section~\ref{general-queries} in order to provide
the desired labeled examples to such an active learning procedure
while using fewer than $m$ queries to our oracle. In the description
below we assume that algorithm $\A$ returns its state, a region $R$ of
the instance space, a number $m$ of desired samples, a boolean flag
$b$ for halting($b=0$) or not ($b=1$), and a classifier $h$.

\begin{algorithm}
{\small
{\bf Input}: The sequence $(x_1, x_2, ..., ) $; allowed error rate $\epsilon$, noise bound $\alpha$, algorithm $\A$.
\begin{list}{\labelitemi}{\leftmargin=1.5em}
\item[1.] Set $b=1$, $t = 1$.  Initialize $\A$ and let $\state(\A)$, $R$, $m$, $b$ and $\hat{h}$ be the returned values.
\item[2.] Let $V$ be a minimal $\epsilon$-cover of $\C$ with respect to the distribution $\D_{X}$.
\item[3.] While $(b)$
\begin{itemize}
\item[(a)] Let $ps = \frac{cd}{\epsilon^2} \log \frac{k}{\epsilon\delta}$ and let $(x_{i_1}, x_{i_2}, \ldots, x_{i_{ps+m}})$ be the first $ps+m$ points in $(x_{t+1},x_{t+2},\ldots)\cap R$.
\item[(b)] Run Phase~\ref{generalized-halving-algorithm} with parameters $\U_1 = (x_{i_1},x_{i_2},\ldots,x_{i_{ps}})$, $V$, $\left\lfloor\frac{1}{16 (\alpha + \epsilon)}\right\rfloor$, $c \log \frac{4 \log_{2} |V|}{\delta^{\prime}}$  
\\Let $h$ be the returned classifier.
\item[(c)] Run Phase~\ref{refining-algorithm} with parameters $\U_2 = (x_{i_{ps+1}},x_{i_{ps+2}},\ldots,x_{i_{ps+m}})$, $h$.\\
Let $L$ be the returned labeled sequence.
\item[(d)] Run $\A$ with parameters $L$ and $\state(\A)$.\\
Let $\state(\A)$, $R$, $m$, $b$ and $\hat{h}$ be the returned values
\item[(e)] Let $t = i_{ps+m}$
\end{itemize}
\end{list}
{\bf Output} Hypothesis $\hat{h}$.
}\caption{{\small General Interactive Algorithm for Bounded Noise}}
\label{alg-noise-a^2}
\label{bounded-noise-alg}
\end{algorithm}

The value $\delta^{\prime}$ in this algorithm should be set appropriately depending on the context,
essentially as $\delta$ divided by a coarse bound on the total number of batches the algorithm $\A$ will request the labels of;
for our purposes a value $\delta^{\prime} = {\rm poly}(\epsilon \delta (1-2\alpha) / d)$ will suffice.
To state an explicit bound on the number of queries used by Algorithm~\ref{alg-noise-a^2},
we first review the following definition of \citet*{Hanneke07,hanneke:thesis}.
Recall that for $r > 0$, we define $B(h,r) = \{g \in \C : \P_{\D_{X}}(x : h(x) \neq g(x)) \leq r\}$.
For any $\H \subseteq \C$, also define the region of disagreement:
$\DIS(\H) = \{x \in \X : \exists h,g \in \H \text{ s.t. } h(x) \neq g(x)\}$.
Then define the \emph{disagreement coefficient} for $h \in \C$ as
\begin{center}
$\theta_{h}(\epsilon) = \sup\limits_{r > \epsilon} \P_{\D_{X}}(\DIS(B(h,r))) / r$.
\end{center}
{\vskip -2mm}Define the disagreement coefficient of the class $\C$ as $\theta(\epsilon) = \sup_{h \in \C} \theta_{h}(\epsilon)$.

\begin{theorem}
\label{thm:disagreement}
For any concept space $\C$ of Natarajan dimension $d$, and any $\alpha \in [0,1/2)$,
for any distribution $\D_{X}$ over $\X$,
\[\SC_{\CCQM}(\epsilon,\delta,\C,\BoundedNoise(\C,\alpha ; \D_{X})) = O\left(\left(1 + \frac{\alpha \theta(\epsilon)}{(1-2\alpha)^2}\right) d k \log^{2}\left(\frac{d k}{\epsilon \delta (1-2\alpha)}\right)\right).\]
\end{theorem}
\begin{proof}[Sketch]
We show that, for $\D_{XY} \in \BoundedNoise(\C,\alpha)$,
running Algorithm~\ref{alg-noise-a^2} with the algorithm $\A$ of~\citet*{Kol10}
returns a classifier $\hat{h}$ with $\err(\hat{h}) \leq \eta + \epsilon$ using
a number of queries as in the claim.

For bounded noise, with noise bound $\alpha$, on each round of Algorithm 4,
we run Algorithm 1 on a set $\U_1$ that, by Hoeffding's inequality and the size of $ps$, with probability $1-\delta/\log(1/\epsilon)$,
has $\min_{h \in V} \err_{\U_1}(h) \leq \alpha + \epsilon$.
Thus, by Lemma~\ref{lem:generalized-halving-algorithm}, the fraction of examples in each $\U_1=(x_{i_{1}},\ldots,x_{i_{ps}})$
on which the returned $h$ makes a mistake is at most $10(\alpha + \epsilon)$.
Then the size of $ps$ and Hoeffding's inequality implies that $\err(h) \leq O(\alpha + \epsilon)$ with probability $1-\delta/\log(1/\epsilon)$,
and a Chernoff bound implies that Algorithm 2 is run on a set $\U_2$ with
$\err_{\U_2}(h) \leq O(\alpha + \epsilon + \sqrt{(\alpha + \epsilon)\log(\log(1/\epsilon)/\delta)/m} + \log(\log(1/\epsilon)/\delta)/m)$.
Thus, by Lemmas \ref{lem:generalized-halving-algorithm} and \ref{lem:good-labels}, the number of queries per round is
$O(k(\alpha + \epsilon) m + k\sqrt{(\alpha + \epsilon) m \log(\log(1/\epsilon)/\delta)} + kd \log(d/\epsilon\delta (1-2\alpha)))$.

In particular, for the algorithm of \citet*{Kol10}, it is known that with probability $1-\delta/2$,
every round has $m \leq O\left(\frac{\theta(\epsilon) d}{(1-2\alpha)^2} \log\left(\frac{1}{\epsilon\delta(1-2\alpha)}\right)\right)$,
 and there are at most $O(\log(1/\epsilon))$ rounds,
so that the total number of queries is at most
$O\left(k\left(\alpha \theta(\epsilon) + 1\right) \frac{d}{(1-2\alpha)^2} \log^2\left(\frac{d}{\epsilon\delta(1-2\alpha)}\right)\right)$.
\end{proof}

The significance of this result is that $\theta(\epsilon)$ is multiplied by $\alpha$, a feature not present in the
known results for active learning.  In a sense, this factor of $\theta(\epsilon)$ is a measure of how difficult
the active learning problem is, as the other terms are inevitable (up to the log factors).

As before, since the value of the noise bound $\alpha$ is typically not known in practice,
it is often desirable to have an algorithm capable of \emph{adapting} to the value of $\alpha$,
while maintaining the query complexity guarantees of Algorithm~\ref{bounded-noise-alg}.
Fortunately, we can achieve this by a similar argument to that used above in Theorem~\ref{thm:adaptive-agnostic}.
That is, starting with an initial guess of $\hat{\alpha}=\epsilon$ as the noise bound argument to Algorithm~\ref{bounded-noise-alg},
we use the budget argument to Phase~\ref{refining-algorithm} to guarantee we never exceed the
query complexity bound of Theorem~\ref{thm:disagreement} (with $\hat{\alpha}$ in place of $\alpha$),
halting early if ever Phase~\ref{refining-algorithm} fails to label the entire $\U_1$ set within its query budget.
Then we repeatedly double $\hat{\alpha}$ until finally this modified Algorithm~\ref{bounded-noise-alg} runs to completion.
Setting the budget sizes and $\delta^{\prime}$ values appropriately, we can maintain the guarantee of Theorem~\ref{thm:disagreement}
with only an extra $\log$ factor increase.

\subsection{Adapting to Unknown $\alpha$}
\label{subsec:unknown-alpha}

Algorithm 4 is based on having direct access to the noise bound $\alpha$.
As in Section~\ref{sec:unknown-rate}, since this information is not typically available in practice,
we would prefer a method that can obtain essentially the same query complexity bounds without
direct access to $\alpha$.  Fortunately, we can achieve this by a similar argument to Section~\ref{sec:unknown-rate},
merely by doubling our guess at the value of $\alpha$ until the algorithm behaves as expected, as follows.

Consider modifying Algorithm 4 as follows.
In Step 6, we include the budget argument to Algorithm 2, with value $O((1+\alpha m) \log(1/\delta^{\prime}))$.
Then, if the set $L$ returned has $|L| < m$, we return Failure.
Note that if this $\alpha$ is at least as large as the actual noise bound,
then this bound is inconsequential, as it will be satisfied anyway (with probability $1-\delta^{\prime}$, by a Chernoff bound).
Call this modified method Algorithm 4$^{\prime}$.

Now consider the sequences $\alpha_{i} = 2^{i-1} \epsilon$, for $1\leq i \leq \log_{2}(1/\epsilon)$.
For $i = 1,2, \ldots, \log_{2}(1/\epsilon)$ in increasing order,
we run Algorithm 4$^{\prime}$ with parameters $(x_1,x_2,\ldots)$, $\epsilon$, $\alpha_i$, $\A$.
If the algorithm runs to completion, we halt and output the $\hat{h}$ returned by Algorithm 4$^{\prime}$.
Otherwise, if the algorithm returns Failure, we increment $i$ and repeat.

Since Algorithm 4$^{\prime}$ runs to completion for any $i \geq \lceil \log(\alpha/\epsilon) \rceil$,
and since the number of queries Algorithm 4$^{\prime}$ makes is monotonic in its $\alpha$ argument,
for an appropriate choice of $\delta^{\prime} = O(\delta \epsilon^{2}/d)$ (based on a coarse bound on the total number of batches the algorithm will request labels for),
we have a total number of queries at most
$O\left((1 + \alpha \theta(\epsilon)) \frac{d}{(1-2\alpha)^2}\log^2\left(\frac{d}{\epsilon \delta(1-2\alpha)}\right) \log\left(\frac{1}{\epsilon}\right)\right)$
for the method of \citet*{Kol10},
only a $O(\log(1/\epsilon))$ factor over the bound of Theorem~\ref{thm:disagreement};
similarly, we lose at most a factor of $O(\log(1/\epsilon))$ for the splitting method,
compared to the bound of Theorem~\ref{thm:splitting-GQM-upper}.

\subsection{Bounds Based on the Splitting Index}
\label{subsec:splitting}

By the same reasoning as in the proof of Theorem~\ref{thm:disagreement},
except running Algorithm~\ref{alg-noise-a^2} with Algorithm~\ref{alg:splitting} instead,
one can prove an analogous bound based on the splitting index of \citet*{sanjoy-coarse}, rather
than the disagreement coefficient.  This is interesting, in that one can also prove a lower bound
on $\SC_{\AL}$ in terms of the splitting index, so that composed with Theorem~\ref{lower-bound-bounded-noise},
we have a nearly tight characterization of $\SC_{\CCQM}(\epsilon,\delta,\D,\BoundedNoise(\C,\alpha ; \D_{X}))$.
Specifically, consider the following definitions due to \citet*{sanjoy-coarse}.

Let $Q \subseteq \{\{h,g\} : h,g \in \C\}$ be a finite set of unordered pairs of classifiers from $\C$.
For $x \in \X$ and $y \in \Y$, define $Q_{x}^{y} = \{\{h,g\} \in Q : h(x) = g(x) = y\}$.
A point $x \in \X$ is said to \emph{$\rho$-split} $Q$ if
\[\max_{y \in \Y} |Q_{x}^{y}| \leq (1-\rho)|Q|.\]
Fix any distribution $\D_{X}$ on $\X$.
We say $\H \subseteq \C$ is \emph{$(\rho,\Delta,\tau)$-splittable} if for all finite
$Q \subseteq \{\{h,g\} \subseteq \C : \P_{\D_{X}}(x : h(x) \neq g(x)) > \Delta\}$,
\[\P_{\D_{X}}(x : x~\rho\text{-splits }Q) \geq \tau.\]
A large value of $\rho$ for a reasonably large $\tau$ indicates that there are highly informative examples
that are not too rare.  Following \citet*{sanjoy-coarse}, for each $h \in \C$, $\tau > 0$, $\epsilon > 0$, we define
\[ \rho_{h,\tau}(\epsilon) = \sup\{ \rho : \forall \Delta \geq \epsilon/2, B(h,4\Delta) \text{ is } (\rho,\Delta,\tau)\text{-splittable}\}.\]
Here, $B(h,r) = \{g \in \C : \P_{\D_{X}}(x : h(x) \neq g(x)) \leq r\}$ for $r > 0$.
Though \citet*{sanjoy-coarse} explores results on the query complexity as a function of $h^*$, $\D_{X}$,
for our purposes (minimax analysis) we will take a worst-case value of $\rho$.
That is, define
\[\rho_{\tau}(\epsilon) = \inf_{h \in \C} \rho_{h,\tau}(\epsilon).\]

Theorem~\ref{lower-bound-bounded-noise} relates the query complexity of $\CCQM$ to that of $\AL$.
There is much known about the latter, and in the interest of stating a concrete result here, we briefly
describe a particularly tight result, inspired by the analysis of \citet*{sanjoy-coarse}.

\begin{lemma}
\label{lem:splitting}
There exist universal constants $c_1,c_2 \in (0,\infty)$ such that,
for any concept space $\C$ of Natarajan dimension $d$, any $\alpha \in [0,1/2)$,
$\epsilon,\delta \in (0,1/16)$, and distribution $\D_{X}$ over $\X$,
\[\inf_{\tau > 0} \frac{c_1}{\rho_{\tau}(4\epsilon)} \leq \SC_{\AL}(\epsilon,\delta,\C,\BoundedNoise(\C,\alpha; \D_{X})) \leq \inf_{\tau > 0} \frac{c_2 d^3}{(1-2\alpha)^2 \rho_{\tau}(\epsilon)} \log^{5}\left( \frac{1}{\epsilon \delta \tau (1-2\alpha)}\right).\]
\end{lemma}

The proof of Lemma~\ref{lem:splitting} is included in Appendix~\ref{app:splitting}.
The implication of the lower bound given by Theorem~\ref{lower-bound-bounded-noise}, combined with Lemma~\ref{lem:splitting} is as follows. 

\begin{corollary}
\label{cor:splitting}
There exists a universal constant $c \in (0,\infty)$ such that,
for any concept space $\C$ of Natarajan dimension $d$, any $\alpha \in [0,1/2)$,
$\epsilon,\delta \in (0,1/32)$, and distribution $\D_{X}$ over $\X$,
\[\SC_{\CCQM}(\epsilon,\delta,\C,\BoundedNoise(\C,\alpha;\D_{X})) \geq \frac{\alpha}{2(k-1)} \cdot \inf_{\tau > 0} \frac{c}{\rho_{\tau}(4\epsilon)} - 4 \ln\left(4\right).\]
\end{corollary}

In particular, this means that in some cases, the query complexity of $\CCQM$ learning is only smaller
by a factor proportional to $\alpha$ compared to the number of random labeled examples required by passive learning,
as indicated by the following example, which follows immediately from Corollary~\ref{cor:splitting} and Dasgupta's
analysis of the splitting index for interval classifiers \citep*{sanjoy-coarse}.

\begin{corollary}
For $\X = [0,1]$ and $\C = \{ 2 \mathbb{I}_{[a,b]} - 1 : a,b \in [0,1]\}$
the class of \emph{interval} classifiers, there is a constant $c \in (0,1)$ such that, for any
$\alpha \in [0,1/2)$ and sufficiently small $\epsilon > 0$,
\[\SC_{\CCQM}(\epsilon,1/32,\C,\BoundedNoise(\C,\alpha)) \geq c \frac{\alpha}{\epsilon}.\]
\end{corollary}

There is also a near-matching upper bound compared to Corollary~\ref{cor:splitting}.
That is, running Algorithm~\ref{alg-noise-a^2} with Algorithm~\ref{alg:splitting} of Appendix~\ref{app:splitting},
we have the following result in terms of the splitting index.

\begin{theorem}
\label{thm:splitting-GQM-upper}
For any concept space $\C$ of Natarajan dimension $d$, and any $\alpha \in [0,1/2)$,
for any distribution $\D_{X}$ over $\X$,
\begin{multline*}
\SC_{\CCQM}(\epsilon,\delta,\C,\BoundedNoise(\C,\alpha;\D_{X})) 
\\ = O\left(kd \log^{2}\left(\frac{d}{\epsilon\delta \tau (1-2\alpha)}\right) + \inf_{\tau > 0} \frac{\alpha k d^3}{(1-2\alpha)^2 \rho_{\tau}(\epsilon)} \log^{5}\left(\frac{1}{\epsilon\delta\tau (1-2\alpha)}\right)\right).
\end{multline*}
\end{theorem}

Logarithmic factors and terms unrelated to $\epsilon$ and $\alpha$ aside, 
in spirit the combination of Corollary~\ref{cor:splitting} with Theorem~\ref{thm:splitting-GQM-upper} 
imply that in the bounded noise model, the specific reduction in query complexity of 
using class-conditional queries instead of label request queries is essentially a factor of $\alpha$.

\section{Other types of queries}
Though the results of this paper are formulated for class conditional queries, similar arguments
can be used to study the query complexity of other types of queries as well.
For instance, as is evident from the fact that our methods interact with the oracle only via the Find-Mistake
subroutine, all of the results in this work also apply (up to a factor of $k$) to a kind of sample-based \emph{equivalence query},
in which we provide a sample of unlabeled examples to the oracle along with a classifier $h$,
and the oracle returns an instance in the sample on which $h$ makes a mistake, if one exists.

\section{Conclusions}
In this paper we propose and study an extension of the standard
  active learning model where more general class-conditional  queries are
  allowed, focusing on the problem of learning in the presence of
  noisy data. We give nearly tight upper and lower bounds on the number of queries needed to learn both for the general agnostic setting and for
the bounded noise model. Our analysis provides a clear picture
into the power of these queries in realistic statistical learning settings,
which may help to inform their use in practical learning problems, as well as provide a point of reference for future exploration of the general
topic of interactive machine learning.

\paragraph{Acknowledgments}
We thank Vladimir Koltchinskii for a number of useful discussions.

This research was supported in part by NSF grant CCF-0953192, ONR grant
N00014-09-1-0751, and AFOSR grant FA9550-09-1-0538.

{\small \bibliography{active}}

\appendix

\section{Useful Facts}
\label{useful}

\begin{lemma}
\label{coins} Let $B_1,\ldots,B_k$ be independent {\rm Geometric($\alpha$)} random variables.
With probability at least $1-\delta$,
\[\sum_{i=1}^{k} B_i \leq \frac{2}{\alpha} \left(k + 4 \ln\left(\frac{1}{\delta}\right)\right).\]
\end{lemma}
\begin{proof}
Let $m = \frac{2}{\alpha}\left(k + 4 \ln\left(\frac{1}{\delta}\right)\right)$.
Let $X_1,X_2,\ldots$ be i.i.d. Bernoulli($\alpha$) random variables.
$\sum_{i=1}^{k} B_i$ is distributionally equivalent to a value $N$
defined as the smallest value of $n$ for which $\sum_{i=1}^{n} X_i = k$,
so it suffices to show $\P(N \leq m) \geq 1-\delta$.

Let $H = \sum_{i=1}^{m} X_i$.
We have $\E[H] = \alpha m \geq 2k$.  By a Chernoff bound, we have
\[\P\left( H \leq k \right) \leq \P\left( H \leq (1/2) \E[H] \right) \leq \exp\left\{-\E[H]/8\right\} \leq \exp\left\{-\ln\left(\frac{1}{\delta}\right)\right\} = \delta.\]
Therefore, with probability $1-\delta$, we have $N \leq m$, as claimed.
\end{proof}

\section{Splitting Index Bounds}
\label{app:splitting}

We prove Lemma~\ref{lem:splitting} in two parts.
First, we establish the lower bound.
The technique for this is quite similar to a result of \citet*{sanjoy-coarse}.
Recall that $\SC_{\AL}(\epsilon,\delta,\C,\Realizable(\C ; \D_{X})) \leq \SC_{\AL}(\epsilon,\delta,\C,\BoundedNoise(\C,\alpha; \D_{X}))$.
Thus, the following lemma implies the lower bound of Lemma~\ref{lem:splitting}.

\begin{lemma}
For any hypothesis class $\C$ of Natarajan dimension $d$, for any distribution $\D_{X}$ over $\X$,
\[\SC_{\AL}(\epsilon,1/16,\C,\Realizable(\C; \D_{X})) \geq \inf_{\tau > 0} \frac{c}{\rho_{\tau}(4\epsilon)}.\]
\end{lemma}
\begin{proof}
The proof is quite similar to that of a related result of \citet*{sanjoy-coarse}.
Fix any $\tau \in (0,1/4)$, and suppose $\A$ is an active learning algorithm that considers at most the first $1/(4\tau)$
unlabeled examples, with probability greater than $7/8$.  Let $h \in \C$ be such that
$\rho_{h,\tau}(4\epsilon) \leq 2\rho_{\tau}(4\epsilon)$, and let
$\Delta \geq 2\epsilon$ and $Q \subseteq \{\{f,g\} \subseteq B(h,4\Delta) : \P_{\D_{X}}(x : f(x) \neq g(x)) > \Delta\}$
be such that $\P_{\D_{X}}(x : x~2\rho_{h,\tau}(4\epsilon)\text{-splits } Q) < \tau$.
In particular, with probability at least $(1-\tau)^{1/(4\tau)} \geq 3/4$, none of the first $1/(4\tau)$
unlabeled examples $2\rho_{h,\tau}(4\epsilon)$-splits $Q$.
Fix any such data set, and denote $\rho = 2 \rho_{h,\tau}(4\epsilon)$.

We proceed by the probabilistic method.
We randomly select the target $h^*$ as follows.
First, choose a pair $\{f^*,g^*\} \in Q$ uniformly at random.
Then choose $h^*$ from among $\{f^*,g^*\}$ uniformly at random.

For each unlabeled example $x$ among the first $1/(4\tau)$,
call the label $y$ with $|Q_{x}^{y}| > (1-\rho)|Q|$ the ``bad'' response.
Given the initial $1/(4\tau)$ unlabeled examples, the algorithm $\A$ has some
fixed (a priori known, though possibly randomized)
behavior when the responses to all of its label requests are the bad responses.
That is, it makes some number $t$ of queries, and then returns some classifier
$\hat{h}$.

For any one of those label requests, the probability that both $f^*$ and $g^*$
agree with the bad response is greater than $1-\rho$.  Thus, by a union bound,
the probability both $f^*$ and $g^*$ agree with the bad responses for the $t$
queries of the algorithm is greater than $1-t \rho$.
On this event, the algorithm returns $\hat{h}$, which is independent from
the random choice of $h^*$ from among $f^*$ and $g^*$.  Since $\P_{\D_{X}}(x : f^*(x) \neq g^*(x)) > \Delta \geq 2\epsilon$,
$\hat{h}$ can be $\epsilon$-close to at most one of them, so that
there is at least a $1/2$ probability that $\err(\hat{h}) > \epsilon$.

Adding up the failure probabilities, by a union bound
the probability the algorithm's returned classifier $h^{\prime}$ has $\err(h^{\prime}) > \epsilon$
is greater than $7/8 - 1/4 - t \rho - 1/2$.  For any $t < 1/(16\rho)$, this is greater than $1/16$.
Thus, there exists some deterministic $h^* \in \C$ for which $\A$ requires at least $1/(16\rho)$
queries, with probability greater than $1/16$.

As any active learning algorithm has a $7/8$-confidence upper bound $M$ on the number
of unlabeled examples it uses, letting $\tau \to 0$ in the above analysis allows $M \to \infty$,
and thus covers all possible active learning algorithms.
\end{proof}

We will establish the upper bound portion of Lemma~\ref{lem:splitting} via the following algorithm.
Here we write the algorithm in a closed form, but it is clear that we could rewrite the method
in the batch-based style required by Algorithm~\ref{bounded-noise-alg} above, simply by including its state every time it
makes a batch of label request queries.
The value $\epsilon_0$ in this method should be set appropriately for the result below;
specifically, we will coarsely take $\epsilon_0 = O((1-2\alpha)^2 \epsilon \tau^2 \delta / d^3)$, based on
the analysis of \citet*{sanjoy-coarse} for the realizable case.

\begin{algorithm}
{\small
{\bf Input}: The sequence $U=(x_1, x_2, ...)$; allowed error rate $\epsilon$; value $\tau \in (0,1)$; noise bound $\alpha \in [0,1/2)$.
\\{\vskip -0mm}I. Let $V$ denote a minimal $\epsilon_0$-cover of $\C$
\\{\vskip -0mm}II. For each pair of classifier $h,g \in V$, initialize $M_{hg} = 0$ 
\\{\vskip -0mm}III. For $T = 1,2,\ldots, \lceil \log_{2}(2/\epsilon) \rceil$
\begin{list}{\labelitemi}{\leftmargin=2em}
\item[1.] Consider the set $Q \subseteq V^2$ of pairs $\{h,g\} \subseteq V$ with $\P_{\D_{X}}(x : h(x) \neq g(x)) > 2^{-T}$
\item[2.] While ($|Q| > 0$)
\begin{itemize}
\item[(a)] Let $S = \emptyset$
\item[(b)] Do $O\left(\frac{1}{(1 - 2 \alpha)^2} \left(d \log \left(\frac{1}{\epsilon}\right) + \log\left(\frac{1}{\delta}\right)\right)\right)$ times
\begin{itemize}
\item[i.] Let $\tilde{Q} = Q$
\item[ii.] While ($|\tilde{Q}| > 0$)
\begin{itemize}
\item[A.] From among the next $1/\tau$ unlabeled examples, select the one $\tilde{x}$ with minimum $\max_{y \in \Y} |\tilde{Q}_{x}^{y}|$, and let $\tilde{y}$ denote the maximizing label
\item[B.] $S \leftarrow S \cup \{\tilde{x}\}$
\item[C.] $\tilde{Q} \gets \tilde{Q}_{\tilde{x}}^{\tilde{y}}$
\end{itemize}
\end{itemize}
\item[(c)] Request the labels for all examples in $S$, and let $L$ be the resulting labeled examples
\item[(d)] For each $h,g \in V$, let $M_{hg} \gets M_{hg} + |\{ (x,y) \in L : h(x) \neq y = g(x)\}|$
\item[(e)] Let $V \gets \left\{ h \in V : \forall g \in V, M_{hg} - M_{gh} \leq O\left(\sqrt{\max\{M_{hg},M_{gh}\} d \log\left(\frac{1}{\epsilon_0}\right)} + d \log\left(\frac{1}{\epsilon_0}\right)\right)\right\}$
\item[(f)] Let $Q \gets \{\{h,g\} \in Q : h,g \in V\}$
\end{itemize}
\end{list}
{\bf Output} Any hypothesis $h \in V$.
}\caption{{\small An active learning algorithm for learning with bounded noise, based on splitting.}}
\label{alg:splitting}
\end{algorithm}

We have the following result for this method, with an appropriate setting of the constants in the ``$O(\cdot)$'' terms.

\begin{lemma}
There exists a constant $c \in (0,\infty)$ such that,
for any hypothesis class $\C$ of Natarajan dimension $d$,
for any $\alpha \in [0,1/2)$ and $\tau > 0$,
for any distribution $\D_{X}$ over $\X$,
for any $\D_{XY} \in \BoundedNoise(\C,\alpha; \D_{X})$,
Algorithm~\ref{alg:splitting} produces a classifier $\hat{h}$ with $\err(\hat{h}) \leq \eta + \epsilon$
using a number of label request queries at most
\[O\left(\frac{d^3}{(1-2\alpha)^2 \rho_{h^*,\tau}(\epsilon)} \log^{5}\left( \frac{1}{\epsilon \delta \tau}\right)\right).\]
\end{lemma}
\begin{proof}[Sketch]
Since $V$ is initially an $\epsilon_0$-cover, the $\hat{h} \in V$ of minimal $\err(\hat{h})$ has $\err(\hat{h}) \leq \epsilon_0$.
Furthermore, $\epsilon_0$ was chosen so that, as long as the total number of unlabeled examples processed
does not exceed $O(\frac{d^2}{(1-2\alpha)^2 \epsilon \tau^2})$, with probability $1-O(\delta)$, we will have
$\hat{h}$ agreeing with $h^*$ on all of the unlabeled examples, and in particular on all of the examples whose
labels the algorithm requests.  This means that, for every example $x$ we request the label of, $\P(\hat{h}(x) = y | x) \geq 1-\alpha$.
By Chernoff and union bounds, with probability $1-O(\delta)$, for every $g \in V$, we always have
\[M_{\hat{h}g} - M_{g\hat{h}} \leq O\left(\sqrt{\max\{M_{hg},M_{gh}\} d \log\left(\frac{1}{\epsilon_0}\right)} + d \log\left(\frac{1}{\epsilon_0}\right)\right),\]
so that we never remove $\hat{h}$ from $V$.
Thus, for each round $T$, the set $V \subseteq B(h^*, 4\Delta_{T})$, where $\Delta_{T} = 2^{-T}$.
In particular, this means the returned $h$ is in $B(h^*, \epsilon)$, so that $\err(h) \leq \eta + \epsilon$.

Also by Chernoff and union bounds, with probability $1-O(\delta)$,
any $g \in V$ with
$M_{\hat{h}g} + M_{g\hat{h}} > O\left( \frac{d}{(1-2\alpha)^2} \log \frac{1}{\epsilon_0}\right)$
has
\[M_{g\hat{h}} - M_{\hat{h}g} > O\left(\sqrt{\max\{M_{hg},M_{gh}\} d \log\left(\frac{1}{\epsilon_0}\right)} + d \log\left(\frac{1}{\epsilon_0}\right)\right),\]
so that we remove it from $V$ at the end of the round.

That $V \subseteq B(h^*,4\Delta_{T})$ also means $V$ is $(\rho,\Delta_{T},\tau)$-splittable, for $\rho = \rho_{h^*,\tau}(\epsilon)$.
In particular, this means we get a $\rho$-splitting example for $\tilde{Q}$ every $\frac{1}{\tau}$ examples (in expectation).
Thus, we always satisfy the $|\tilde{Q}|=0$ condition after at most
$O\left(\frac{d}{\rho} \log^{2}\frac{1}{\epsilon_0}\right)$
rounds of the inner loop (by Chernoff and union bounds, and the definition of $\rho$).
Furthermore, among the examples added to $S$ during this period, regardless of their true labels we are guaranteed that
at least $1/2$ of pairs $\{h,g\}$ in $Q$ have at least one of $(M_{h\hat{h}}+M_{\hat{h}h})$ or $(M_{g\hat{h}}+M_{\hat{h}g})$ incremented as a result:
that is, for at least $|Q|/2$ pairs, at least one of the two classifiers disagrees with $\hat{h}$ on at least one of these examples.
Thus, after executing this
$O\left(\frac{1}{(1 - 2 \alpha)^2} d \log \left(\frac{1}{\epsilon_0}\right)\right)$ times,
we are guaranteed that at least half of the $\{h_1,h_2\}$ pairs in $Q$ have (for some $i \in \{1,2\}$)
$M_{\hat{h}h_i} + M_{h_i\hat{h}} > O\left( \frac{d}{(1-2\alpha)^2} \log \frac{1}{\epsilon_0}\right)$,
thus reducing $|Q|$ by at least a factor of $2$.  Repeating this $\log |Q| = O(d \log (1/\epsilon_0))$ times satisfies the $|Q|=0$ condition.

Thus, the total number of queries is at most
\[O\left(\frac{1}{(1-2\alpha)^2}\frac{d^3}{\rho} \log^{5}\frac{1}{\epsilon_0}\right).\]
\end{proof}

\end{document}